\documentclass[12pt]{l4dc2023}

\title[Probabilistic Verification of ReLU NNs via CFs]{Probabilistic Verification of ReLU Neural Networks\\ via Characteristic Functions}
\usepackage{times}
\usepackage{upgreek}
\usepackage{mathrsfs}
\usepackage{mathtools}
\usepackage{cases}
\usepackage{nccmath}
\usepackage{adjustbox}
\usepackage{tikz}
\usepackage{bigdelim}
\usepackage{algorithm}
\usepackage{enumitem}

\newtheorem{thm}{Theorem}
\newtheorem{corr}{Corollary}

\newtheorem{prop}{Proposition}
\newtheorem{defn}{Definition}

\newcommand{\U}{\mathrm{U}}

\author{
\Name{Joshua Pilipovsky}  \Email{jpilipovsky3@gatech.edu}\\
\addr Daniel Guggenheim School of Aerospace Engineering, Georgia Institute of Technology
\AND
\Name{Vignesh Sivaramakrishnan} \Email{vigsiv@unm.edu}\\
\addr Department of Electrical and Computer Engineering, University of New Mexico%
\AND
\Name{Meeko M. K. Oishi} \Email{oishi@unm.edu}\\
\addr Department of Electrical and Computer Engineering, University of New Mexico%
\AND
\Name{Panagiotis Tsiotras} \Email{tsiotras@gatech.edu}\\
\addr Daniel Guggenheim School of Aerospace Engineering, Georgia Institute of Technology
}

\begin{document}

\maketitle

\begin{abstract}
Verifying the input-output relationships of a neural network so as to achieve
some desired performance specification is a difficult, yet important, problem due to the growing ubiquity of neural nets in many engineering applications. 
We use ideas from probability theory 
in the frequency domain to provide probabilistic verification guarantees for ReLU neural networks.
Specifically, we interpret a (deep) feedforward neural network as a discrete dynamical system over a finite horizon that shapes distributions of initial states, and use characteristic functions to propagate the distribution of the input data through the network.
Using the inverse Fourier transform, we obtain the corresponding cumulative distribution function of the output set, which can be used to check if the network is performing as expected given any random point from the input set.
The proposed approach does not require distributions to have well-defined moments or moment generating functions.
We demonstrate our proposed approach on two examples, and compare its performance to related approaches.
\end{abstract}


\section{Introduction}

Neural networks (NN) have become a powerful tool in recent years for a large class of applications, including image classification \citep{imageClassNN}, speech recognition \citep{speechNN}, autonomous driving \citep{selfDrivingNN}, drone acrobatics \citep{droneAcrobatics}, and many others.
The formal verification of neural networks is crucial for their wider adoption in safety-critical scenarios. 
The main difficulty with the use of (deep) NN for safety-critical applications 
lies in the demonstrated sensitivity of DNNs to input uncertainties and/or adversarial attacks.
For example, in the context of image classification, adding even a small amount of noise to the input set can greatly change the network output~\citep{classificationAttack,adversaryNN}.
For safety-critical applications, DNNs should be robust or insensitive to input uncertainties, a property that be tested by verifying that the network prescribes to certain output specifications subject to various inputs. 

\begin{figure}
    \centering
    \includegraphics[width=0.8\textwidth]{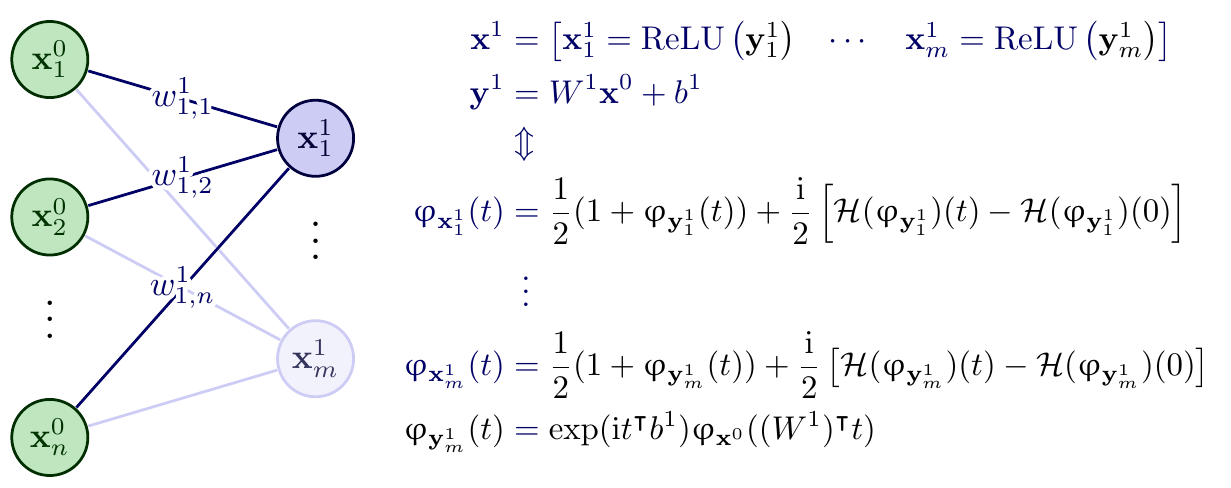}
    \caption{%
    The characteristic function of the input data can be propagated through a ReLU network analytically.
    This enables one to query the characteristic function (CF) of the network to answer out-of-distribution
    questions at the output.
    The use of CFs also circumvents difficulties in cases where the underlying distributions do not have any moments or moment-generating functions (e.g., Cauchy distribution).}
    \label{fig:high_level}
\end{figure}

Verification frameworks for DNNs can be classified as either {\em exact} or {\em probabilistic}. 
In exact verification, a deterministic input set is mapped to an output set; if any output falls outside the safety set, the verification fails.
This is referred to as \textit{worst-case} safety verification since the input set can be treated as an uncertainty set centered around some nominal input.
Given some input $x_0$ and a neural network $f:x\mapsto y$, exact verification can be posed as a nonlinear program (NLP), with the objective function quantifying satisfaction of some safety rule $y\in\mathcal{S}$.
In general, though, the resulting NLP is intractable using standard off-the-shelf solvers.
Several works have used mixed-integer linear programming (MILP) \citep{MILPNN1,MILPNN2}, Satisfiability Modulo Theories (SMT) \citep{SMT1,SMT2}, or semi-definite programming (SDP) \citep{VERIF_PAVONE,exactVerification_fazlyab,exactVerification_LP2,exactVerification_SDP2,exactVerification_LP1,exactVerification_SDP1}, to recast and solve this NLP problem.
In recent work, given an input or an output polytope, one can generate the respective output or input polytope through the ReLU neural network~\citep{9561956}.

In probabilistic verification, the input set itself is uncertain and potentially unbounded.
Random uncertainties naturally arise in practical applications, for example, from signal processing, environmental noise, and other exogenous disturbances.
In this context, the uncertainties are modeled in terms of probability distributions, and the verification problem is to find the \textit{probability} that the output is contained in a safety set given a random input from the input set.
Given a random input vector $\mathbf{x}_0$ 
and a neural network $f$, the probability that the output random vector $\mathbf{y} = f(\mathbf{x}_0)$ lies in some safety set $\mathcal{S}$ is greater than some threshold $1 - p$ is given by the {chance} constraint
\begin{equation}
    \label{eq:CC}
    \mathbb{P}(\mathbf{y}\in\mathcal{S}) \geq 1 - p.
\end{equation}
Relatively few works have studied the verification of DNNs in a probabilistic setting; most of the existing approaches involve under- or over-approximations.
In \cite{probVerification_fazlyab}, an output confidence ellipsoid is estimated via an SDP that is an affine and quadratic relaxation, and then equivalence between confidence sets and chance constraints is used to solve the verification problem.
PROVEN \citep{PROVEN} accommodates bounded disturbances, using linear approximations of activation functions and concentration inequalities to generate bounds on (\ref{eq:CC}).  
In \cite{Pautov_2022}, a similar approach is taken with Cramer-Chernoff concentration inequalities, but because it is based on sampling, a linear approximation of the activation functions is not needed. 
Generative DNNs are considered in \cite{probVerification_deepmind}, which formulates an upper bound on the chance constraint via duality.
Lastly, a scenario optimization approach in \cite{9857969} constructs a lower bound on \eqref{eq:CC} that depends upon the number of samples.

In this paper, we interpret a DNN as a dynamical system~\citep{NNdynamicalsystems1, NNdynamicalsystems2,E2017} that shapes distributions of data
and view the verification problem as one of 
propagating a distribution through a linear stochastic system to form an output distribution that needs to meet the safety constraints.
We focus on DNNs with rectified linear unit (ReLU) activation functions, as the piece-wise linear nonlinearity of ReLU allows us to \textit{analytically} propagate the distribution without any loss of accuracy.
Given an input distribution to the network, we compute its characteristic function and use its properties to derive the output distribution, from which we can verify the output chance constraint. 
Therefore, we can provide rigorous statistical guarantees for the performance of any given ReLU neural network for any input distribution.

The paper is organized as follows. 
Section~\ref{sec:probForm} introduces the preliminaries and problem formulation.
Section~\ref{sec:CFtheory} presents the main properties of characteristic functions we use in our work and states the main result that allows us to propagate a characteristic function through a ReLU neural network.
Section~\ref{sec:verification} presents the safety verification algorithm given the machinery developed in the previous section applied to output polytopes.
Examples demonstrating the theory are given in Section~\ref{sec:examples}, and we provide some concluding remarks and avenues for future work in Section~\ref{sec:conclusion}.

\section{Preliminaries and Problem Formulation} \label{sec:probForm}

\subsection{Notation}
Real-valued vectors are denoted by lowercase letters, $u \in \mathbb{R}^m$, matrices are denoted by uppercase letters, $V \in \mathbb{R}^{n \times m}$, and random vectors are denoted by boldface, $\mathbf{w}\in\mathbb{R}^{p}$.
We denote the $j^{th}$ component of a vector by the subscript $u_j$, and the $i^{th}$ row and $j^{th}$ column of a matrix by $V_{i,j}$.
The imaginary unit is denoted by $\mathrm{i} := \sqrt{-1}$.
The $d$-dimensional vector $e_{i,d} = [0 \ \cdots \ 1 \ \cdots \ 0]^\intercal \in\mathbb{R}^{d}$ is a basis vector which selects the $i^{th}$ element of a vector $\psi\in\mathbb{R}^{d}$ via $\psi_{i} = e_{i,d}^\intercal \psi$.
A random vector $\mathbf{w}$ is defined on the probability space $(\Omega,\mathscr{B}(\Omega),\mathbb{P}_{\mathbf{w}})$~\cite[Sec.~2]{billingsley_probability_2012}.
We only consider continuous random vectors, i.e., those having probability measure $\mathbb{P}_{\mathbf{w}}\left(\{\mathbf{w} \in \mathcal{S}\}\right) = \int_{\mathcal{S}} \psi_{\mathbf{w}}(z)\, \mathrm{d} z$ for $\mathcal{S}\subseteq\mathscr{B}(\Omega)$, and PDF $\psi_{\mathbf{w}}$ that satisfies $\psi_{\mathbf{w}}\geq 0$ almost everywhere (a.e.) such that $\int_{\mathbb{R}} \psi_{\mathbf{w}}(z)\, \mathrm{d} z = 1$.
For the random variable $\mathbf{y} = a^\intercal\mathbf{w},\ a\in\mathbb{R}^p$, we characterize the probability $\mathbb{P}\{a^\intercal\mathbf{w} \leq \alpha \}$ using the cumulative distribution function (CDF) $\Phi_{a^\intercal\mathbf{w}}:\mathbb{R}\rightarrow[0,1]$, that is, by $\mathbb{P}\{a^\intercal\mathbf{w} \leq \alpha \} =  \Phi_{a^\intercal\mathbf{w}}(\alpha)$~\cite[Sec.~14]{billingsley_probability_2012}. 
We write $\mathbf{w} \sim \psi_{\mathbf{w}}$ to denote the fact that $\mathbf{w}$ is distributed according to the PDF $\psi_{\mathbf{w}}$.
We denote a uniform distribution as $\U[a,b]$ where $a,b\in\mathbb{N},\
a<b$.

\subsection{Problem Formulation}

We consider an $L$-layer ReLU DNN with input $\mathbf{x}^0 \in \mathbb{R}^{h_{0}}$ and output $\mathbf{y} = f(\mathbf{x}^0) = \mathbf{x}^l\in\mathbb{R}^{h_{L}}$, with $f$ being the composition of $L$ layers, that is, $f = f_{L-1} \circ \cdots \circ f_0$.
The $k$th layer of the ReLU network corresponds to a function $f_k: \mathbb{R}^{h_{k}} \rightarrow \mathbb{R}^{h_{k+1}}$ of the form
\begin{equation}
    \mathbf{x}^{k+1} = f_k(\mathbf{x}^k) = \sigma(W^k \mathbf{x}^k + b^k), 
\end{equation}
where $W^k \in \mathbb{R}^{h_{k+1}\times h_k}$ is the weight matrix, $b^k \in\mathbb{R}^{h_{k+1}}$ is the bias, and $\sigma(x^k_j) := \textrm{max}(0,x^k_j)$ is the component-wise ReLU function, where $x^k_j$ is the $j$th component of $x^k\in\mathbb{R}^{h_k}$.
We assume that the last layer is an affine transformation, that is, $\mathbf{x}^L = W^{L-1} \mathbf{x}^{L-1} + b^{L-1}$.
Note that convolution layers can be captured by this framework, as they correspond to linear layers $W^k$ endowed with a particular matrix structure. 

Let the mapping $f : \mathcal{X} \mapsto \mathcal{Y}$ with 
$\mathcal{X}$ and $\mathcal{Y}$ subsets of Euclidean spaces of given dimensions, and let $\mathcal{S} \subset \mathcal{Y}$ denote the output safety set.
We would like to answer the following questions: \
\begin{itemize}
    \item Given a random sample from the input set $x = \mathbf{x} (\omega) \in \mathcal{X}$, where $\omega \in \Omega$,  what is the probability that the output $y = f(x) \in \mathcal{Y}$ lies in the output set $\mathcal{S}$?
    Equivalently, given some verification threshold $p\in(0,1]$, is the chance constraint (\ref{eq:CC}) satisfied for all $x\in\mathcal{X}$?
    \item Given the numerically computed output distribution $\hat{\psi}_{\mathbf{y}}$, what is the relative error in the probability of satisfaction of the output chance constraint compared to that of the true output distribution $\psi_{\mathbf{y}}$?
\end{itemize}

To answer the above questions, we use the machinery of characteristic functions (CF) to propagate a distribution through a ReLU network allowing us to perform the verification task.

\section{Characteristic Functions}\label{sec:CFtheory}

We assume that the input distribution $\psi_0$ over the input set $\mathcal{X}$ is given.
The analog of the probability density function $\psi_{\mathbf{x}}$ in the spatial domain is the characteristic function $\upvarphi_{\mathbf{x}}$ in the frequency domain.

\begin{defn}[Characteristic Function]
For a continuous random vector $\mathbf{w}\in\mathbb{R}^{p}$ such that $\mathbf{w} \sim \psi_{\mathbf{w}}$, the characteristic function (CF) is the Fourier transform $\mathcal{F}(\psi_{\mathbf{w}})(t)$ of the PDF $\psi_{\mathbf{w}}\in\mathcal{L}^2(\mathbb{R}^p)$ given by
\begin{align}
    \label{eq:CFdefn}
    \upvarphi_{\mathbf{w}}(t) := \mathbb{E}_{\mathbf{w}}\left[e^{\mathrm{i} t^\intercal\mathbf{w}}\right] = \mathcal{F}(\psi_{\mathbf{w}})(t) = \int_{\mathbb{R}^p}e^{\mathrm{i} t^\intercal z} \psi_{\mathbf{w}}(z)\, \mathrm{d} z,
\end{align} 
where $t\in\mathbb{R}^{p}$.
\end{defn}
The CF has the following properties~\citep{cramer_mathematical_1999,lukacs_characteristic_1970}: 
\begin{enumerate}[label={P\arabic*}:]
\itemsep=0pt
    \item It is uniformly continuous.
    
    \item It is bounded, i.e., $|\upvarphi_{\mathbf{w}}(t)| \leq 1$, $\forall t \in \mathbb{R}^{p}$.
    
    \item It is Hermitian, i.e., $\upvarphi_{\mathbf{w}}(-t) = \overline{\upvarphi}_{\mathbf{w}}(t)$, where $\bar\upvarphi$ denotes the complex conjugate of $\upvarphi$. 
    
    \item Let $\mathbf{w}_1,\mathbf{w}_2$ be random vectors of appropriate dimensions and let $\mathbf{z} = \mathbf{w}_1 + \mathbf{w}_2$.
    Then, $\psi_{\mathbf{z}}(z) = \big(\psi_{\mathbf{w}_1} * \psi_{\mathbf{w}_2}\big)(z)$ (i.e., convolution of their PDFs), and $\upvarphi_{\mathbf{z}}(t) = \upvarphi_{\mathbf{w}_1}(t)\upvarphi_{\mathbf{w}_2}(t)$.
    
    \item Given $\mathbf{z} = F\mathbf{w} + g$ for $F \in \mathbb{R}^{n\times p}, \ g\in\mathbb{R}^{n}$, the CF is $\upvarphi_{\mathbf{z}}(t) = e^{\mathrm{i} t^\intercal g} \upvarphi_{\mathbf{w}}(F^\intercal t)$.
   
    \item Given two scalar and independent random variables $\mathbf{w}_1, \mathbf{w}_2\in\mathbb{R}$, then $\mathbf{z} = [\mathbf{w}_1, \mathbf{w}_2]^\intercal \in \mathbb{R}^2$ has the PDF $\psi_{\mathbf{z}}(z) = \psi_{\mathbf{w}_1}(e_1^\intercal z) \psi_{\mathbf{w}_2}(e_2^\intercal z) = 
    \psi_{\mathbf{w}_1}(z_1) \psi_{\mathbf{w}_2}(z_2)$ where $z = [z_1, z_2]^\intercal \in \mathbb{R}^2$, and  $\upvarphi_{\mathbf{z}}(t) = \upvarphi_{\mathbf{w}_1}(e_1^\intercal t)\upvarphi_{\mathbf{w}_2}(e_2^\intercal t) =   \upvarphi_{\mathbf{w}_1}(t_1)\upvarphi_{\mathbf{w}_2}(t_2)$, where $t=[t_1,t_2]^\intercal \in \mathbb{R}^2$.
    
\end{enumerate}

We note that the characteristic function of a distribution always exists, even when the probability density function or moment-generating function do not exist.
We can recover the CDF of a distribution from its CF using the Hilbert transform.

\begin{defn}
The Hilbert transform (HT) of a function $f$ is defined as the linear integral operator
\begin{equation}
    \label{eq:HTdefn}
    \mathcal{H} (f)(t) := \frac{1}{\pi} \, \mathrm{p.v.} \int_{\mathbb{R}} \frac{f(\tau)}{t - \tau} \, \mathrm{d} \tau,
\end{equation}
where \emph{p.v.} denotes the Cauchy principal value, that is, 
\begin{equation}
    \label{eq:pv}
    \mathrm{p.v.} \int_{\mathbb{R}} f(t) \, \mathrm{d} t = \lim_{\epsilon\downarrow 0, a\uparrow\infty}
    \left[
    \int_{\epsilon}^{a}  f(t) \, \mathrm{d} t
    +
    \int_{-a}^{-\epsilon}  f(t) \, \mathrm{d} t
    \right].
\end{equation}
\end{defn}

If not stated otherwise, all integrals in this paper are understood in the principal value sense.
Using the change of variables $\tau \rightarrow -\tau$, one may equivalently express the HT as 
\begin{equation}
    \label{eq:HTdefn2}
    \mathcal{H} (f)(t) = \frac{1}{\pi} \int_{0}^{\infty} [f(t - \tau) - f(t + \tau)] \ \frac{\ \mathrm{d} \tau}{\tau}.
\end{equation}
Note that the Hilbert transform is bounded on $\mathcal{L}^2(\mathbb{R})$ ~\citep{MR2933659}.
To numerically compute the Hilbert transform of a continuous function, we use a finite expansion of the sinc function based on the work of \cite{sincHT}, which is parameterized by the resolution $h$ and the number of terms $M$ in the expansion.
\begin{thm}[Gil-Pelaez Inversion Theorem, \citep{sincHT,gil-pelaez_note_1951}]
\label{thm:GPI}
\newline
Given a random variable $\mathbf{x}$ with CF $\upvarphi_{\mathbf{x}}$, the CDF of $\mathbf{x}$, $\Phi_{\mathbf{x}}(\cdot)$, at each point of continuity $x$, can be evaluated by
\begin{equation}
    \label{eq:cftocdf}
    \Phi_{\mathbf{x}}(x) = \frac{1}{2}-\frac{\mathrm{i}}{2}\mathcal{H}(e^{-\mathrm{i}t x} \,\upvarphi_{\mathbf{x}}(t))(0),
\end{equation}
where $x,t\in\mathbb{R}$.
\end{thm}

\subsection{Propagation of a Characteristic Function through a ReLU Network}
\label{sec:ReLU_CF}

Given an initial characteristic function $\upvarphi^{0}$ that represents the input distribution, we compute the output characteristic function $\upvarphi^{L}$.
At an arbitrary layer $k$ this propagation can be split into a two-step process: (i) propagate the CF through the affine layer to obtain $\upvarphi_{\mathbf{y}^k}$, where
$\mathbf{y}^k = W^k \mathbf{x}^k + b^k$,
and (ii) propagate the intermediate CF through the ReLU layer to obtain the output $\upvarphi_{\mathbf{x}^{k+1}}$. 
Using Property P2 of CFs, it is straightforward to compute
\begin{equation}
    \label{eq:ILindividual}
    \upvarphi_{\mathbf{y}^k}(t) = \exp(\mathrm{i} t^\intercal b^k) \upvarphi_{\mathbf{x}^k}((W^k)^\intercal t).
\end{equation}

Given the intermediate CF $\upvarphi_{\mathbf{y}^k}$, we can compute the component-wise CF after the ReLU, based on the work of \cite{CF_PINELIS}.
To this end,
let $\mathbf{x}\in\mathbb{R}$ be a scalar random variable, and introduce the operator
\begin{equation}
    \label{eq:Jdefn}
    J_a (\upvarphi)(t) := \frac{1}{2\pi\mathrm{i}}\int_{\mathbb{R}} e^{-\mathrm{i}a\eta} \upvarphi(t + \eta) \, \frac{\mathrm{d} \eta}{\eta}.
\end{equation}
Using the change of variables $u \mapsto -u$, one may equivalently write (\ref{eq:Jdefn}) as 
\begin{equation}
    \label{eq:Jdefn2}
    J_a(\upvarphi)(t) = \frac{1}{4\pi \mathrm{i}} \int_{\mathbb{R}} [e^{-\mathrm{i}a\eta} \upvarphi(t + \eta) - e^{\mathrm{i}a\eta}\upvarphi(t - \eta)] \, \frac{\mathrm{d} \eta}{\eta}.
\end{equation}
\begin{prop}
Let $\mathbf{x}\in\mathbb{R}$ be a real-valued random variable with characteristic function $\upvarphi_{\mathbf{x}}$.
Then, 
\begin{equation}
    \label{eq:prop1}
    J_{a} (\upvarphi_{\mathbf{x}})(t) = \frac{1}{2} \mathbb{E}[e^{\mathrm{i}t\bf{x}} \ \mathrm{sgn}(\mathbf{x} - a)].
\end{equation}
\end{prop}
\begin{proof}
The proof is straightforward using the definition of $J_{a}$ in (\ref{eq:Jdefn}) and Fubini's theorem. See \citep{CF_PINELIS} for details.
\end{proof}

Next, consider the ReLU operator $\mathrm{ReLU}(x) = \max(0, x)$. 
Using the identity 
\begin{equation}
    \label{eq:reluIdentity}
    2e^{\textrm{i}t\max(0,x)} = 1 + e^{\mathrm{i}tx} + e^{\mathrm{i}tx} \ \mathrm{sgn}(x) - \mathrm{sgn}(x),
\end{equation}
we can derive the CF of the ReLU operator.
\begin{corr}
The characteristic function of the random variable $\mathbf{x}_{+} := \max(0,\mathbf{x})$ is given by
\begin{align}
    \label{eq:reluCF}
    \upvarphi_{\mathbf{x}_{+}}(t) := \mathbb{E}[e^{\mathrm{i}t\mathbf{x}_{+}}] = \frac{1}{2}\left[1 + \upvarphi_{\mathbf{x}}(t)\right] + J_0(\upvarphi_{\mathbf{x}})(t) - J_0(\upvarphi_{\mathbf{x}})(0).
\end{align}
\end{corr}
\begin{proof}
From the identity in (\ref{eq:reluIdentity}), take the expectation of both sides, which yields
\begin{equation}
    2\mathbb{E}[e^{\mathrm{i}t\mathbf{x}_{+}}] = 1 + \mathbb{E}[e^{\mathrm{i}t\mathbf{x}}] + \mathbb{E}[e^{\mathrm{i}t\mathbf{x}} \ \mathrm{sgn}(\mathbf{x})] - \mathbb{E}[\mathrm{sgn}(\mathbf{x})].
\end{equation}
Using (\ref{eq:prop1}) from Proposition~1 with $a = 0$ we get the desired result (\ref{eq:reluCF}).
\end{proof}
Comparing the alternative definition of $J_a$ in (\ref{eq:Jdefn2}) with the alternative definition of the HT in (\ref{eq:HTdefn2}) we can identify $J_0 = \frac{\mathrm{i}}{2}\mathcal{H}$, which implies that
\begin{align}
    \label{eq:step2Propagation}
    &\upvarphi_{\mathbf{x}^{k+1}_j}(t_j) =\frac{1}{2}(1 + \upvarphi_{\mathbf{y}^k_j}(t_j)) + \frac{\mathrm{i}}{2}\left[\mathcal{H}(\upvarphi_{\mathbf{y}^k_j})(t_j) - \mathcal{H}(\upvarphi_{\mathbf{y}^k_j})(0)\right],
\end{align}
where $t_j = e^\intercal_{j,h_{k+1}}t$ isolates the $j$th element of the frequency variable $t$.

\section{Probabilistic DNN Verification}\label{sec:verification}

With the developed CF machinery outlined in Section~\ref{sec:ReLU_CF}, we can verify ReLU networks to a prescribed degree of accuracy.
For example, if $p = 0.05$, then a NN passes verification if \textit{at least} 95\% of the input samples belong in the desired output set $\mathcal{S}$.
We presume that the output set $\mathcal{Y} \subseteq \mathbb{R}^{h_L}$ can be represented by a convex polytope, that is, an intersection of halfspaces.  For notational simplicity, we consider an output set that can be written as 
\begin{equation}  \label{eq:lincon}
    \mathcal{S} = \{y\in\mathbb{R}^{h_l} \ | \ c^\intercal y \leq d\}.
\end{equation}%
and note that generalization to convex polytopes follows easily by analyzing each half-space independently. 

\begin{minipage}[c]{0.92\textwidth}
\LinesNumbered
\begin{algorithm}[H]
\caption{ReLU Network Verification}\label{alg:cap}
\SetKwInOut{Input}{Input}
\SetKwInOut{Output}{Output}
\Input{$\upvarphi_{\mathbf{x}}$, $\{c,d\}$, $N, M, h, p$}
\Output{$\hat{p}$, pass/fail}
$\upvarphi_{\mathbf{x}^0_{j}}$ $\gets$ Compute initial CF components on grid

$\upvarphi_{\mathbf{x}^L_{j}}$ $\gets$ Propagate through ReLU network using (\ref{eq:ILindividual}) and (\ref{eq:step2Propagation})

$\upvarphi_{\mathbf{y}}$ $\gets$ Compute CF of output r.v. $\mathbf{y} := c^\intercal \mathbf{x}_{L}$

$\Phi_{\mathbf{y}}$ $\gets$ Compute CDF using (\ref{eq:cftocdf})

$\hat{p} := \mathbb{P}(\mathbf{y}\in\mathcal{S}) = 1-\Phi_{\mathbf{y}}(d)$

\If{$\hat{p} < 1 - p$}
{Fail verification
\lElse{Pass verification}
}
\end{algorithm}
\end{minipage}

The first three parameters the algorithm accepts are the characteristic function of the input, $\upvarphi_{\mathbf{x}}$, and the parameters that define the half-space, $c,d$.
The last two design choices are the HT resolution, specified by $N, h, M$, and the cutoff probability for verification, $p$.
The initial CF is then propagated through the network, which yields the final CF. 
Since the output set is a half-space, the probability for the output $\mathbf{x}^{L}$ to be in the half-space is given by
\begin{equation}
    \mathbb{P}(\mathbf{x}^{L} \in \mathcal{S}) = \Phi_{\mathbf{y}}(d) = \frac{1}{2} - \frac{\mathrm{i}}{2}\mathcal{H}(e^{-\mathrm{i}td}\upvarphi_{\mathbf{y}}(t))(0),
\end{equation}
where $\mathbf{y} := c^\intercal\mathbf{x}^{L}$ and $\upvarphi_{\mathbf{y}}(t) = \prod_{j} \upvarphi_{\mathbf{x}^{L}_{j}}(c_{j}t)$.
Thus, Steps 3-4 in Algorithm~1 compute the associated CF and CDF of the constraint \eqref{eq:lincon}.
The CDF evaluated at $x = d$ represents the probability of the event $\{c^\intercal\mathbf{x} \leq d\}$; if this value is less than $1 - p$, this is below the cutoff for verification. 
As an example, if $\Phi_{c^\intercal\mathbf{x}^{L}}(d) = 0.7$ but $p = 0.1$, then only 70\% of samples from the output set lie in the safety set, which is less than the cutoff of 90\%; hence the verification test fails in this case.

\section{Examples} \label{sec:examples}

We provide two examples that illustrate the proposed verification algorithm.
Both examples use ReLU feedfoward neural networks from the verification literature.
All simulations were run on a 32~GB Intel i7-10750H @ 2.60~GHz computer.
For computations and memory storage, we use python with JAX \citep{jax2018github}. 
JAX was run on CPU-only mode but can be run on GPUs or TPUs.
All trials of the verification algorithm were compared to an empirical truth computed by brute-force propagation of $10^4$ samples through the ReLU networks for each example.

\subsection{Example 1}

\begin{figure}[htb]
    \centering
    \includegraphics[width=\textwidth,trim={0.1cm 0cm 0.1cm 0},clip]{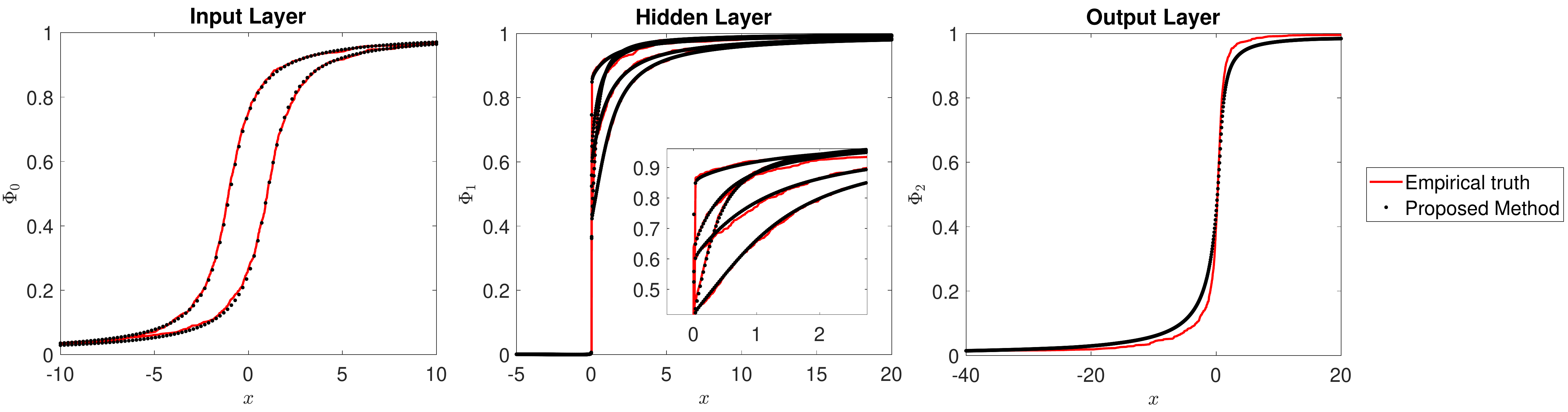}
    \caption{The characteristic function and CDF for each layer in the ReLU network. 
    The CDF computed using the proposed method (black dot) using (\ref{eq:cftocdf}) closely resembles the empirical CDF computed from brute-force propagation of $10^4$ input samples (red line).}
    \label{fig:CDFsComparison}
\end{figure}

To showcase the proposed verification scheme and its advantages, we simulated the following scenario, adapted from \cite{VERIF_PAVONE}. 
We run a set of 1000 trials, where the ReLU network weights and biases are uniformly sampled from $\U[-1, 1]$ for various parameters affecting the accuracy of the CF propagation.
The network architecture has two inputs, one output, and one hidden layer with 10 neurons.
The output safety set is $\mathcal{S} = \{x_{L} : x_{L} \geq 0\}$.
The maximum probability of lying outside the safety set is $p = 0.05$.
Lastly, we also generated (an approximation of) the true output set $\mathcal{Y}$ by propagating 1 million samples from the input set through the network.
The inputs are modeled as Cauchy distributions with CF
\begin{equation}
    \upvarphi_{0}(t) = \textrm{exp}(x_0 \mathrm{i} t - \gamma |t|),
\end{equation}
with locations $x_0^{(1)} = 1, x_0^{(2)} = -1$ and scale $\gamma^{(1)} = \gamma^{(2)} = 1$.
The CF of a distribution allows one to easily compute its moments from the derivatives of the CF via the following expression
\begin{equation}
    \mathbb{E}[\mathbf{x}^k] = \mathrm{i}^{-k}\upvarphi_{\mathbf{x}}^{(k)}(0), \label{eq:momentsFromCF}
\end{equation}
where $f^{(k)}$ denotes the $k$th derivative of the function $f$.
We emphasize that the derivatives of the Cauchy CF do \textit{not} exist at zero, hence this distribution does not have any standard moments nor does it have a moment generating function.
As a result, the methods proposed in \cite{Pautov_2022, probVerification_fazlyab,PROVEN} would not work in this case.

To show how the distribution of the inputs propagates throughout the ReLU network, we take a snapshot of the CFs and CDFs for a few random trials. 
The plots in Figures~\ref{fig:CDFsComparison}-\ref{fig:CDFprobabilities} correspond to the parameters $\{N,h,M,d\} = \{10000, 0.05, 5000, 50\}$, namely, 10001 terms in the HT for each of the $10^4$ grid points in the domain $\mathcal{D} = \{t : t \in [-50, 50]\}$.
Figure~\ref{fig:CDFsComparison} shows the CDF at each layer in the network, as computed from the CF through the HT.
It closely resembles the \textit{ground-truth} CDF computed via sampling.

\begin{figure}[t!]
    \centering
    \subfigure[For this random trial, our method accurately determines the violation of the output safety set. 
    Our results (black) are very close ($|\delta\Delta| = 0.049$) to the empirically obtained CDF and likelihood (red).][t]{%
\label{fig:errorProb}\includegraphics[width=0.48\textwidth]{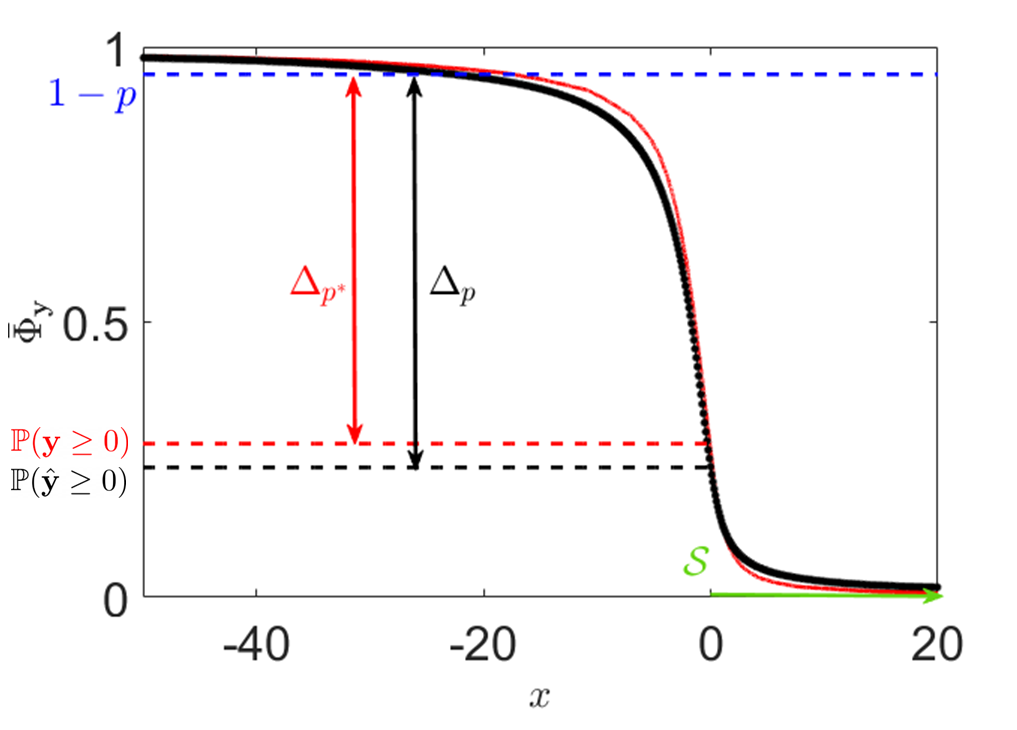}}
    \hspace{0.25cm}
    \subfigure[Our estimated safety set at the desired probability threshold (black) is much closer to the empirically determined safety set (red) as compared to other SoTA methods (magenta).][t]{\label{fig:outputQuantile}\includegraphics[width=0.465\textwidth,trim={0cm 0cm 0cm 0cm}]{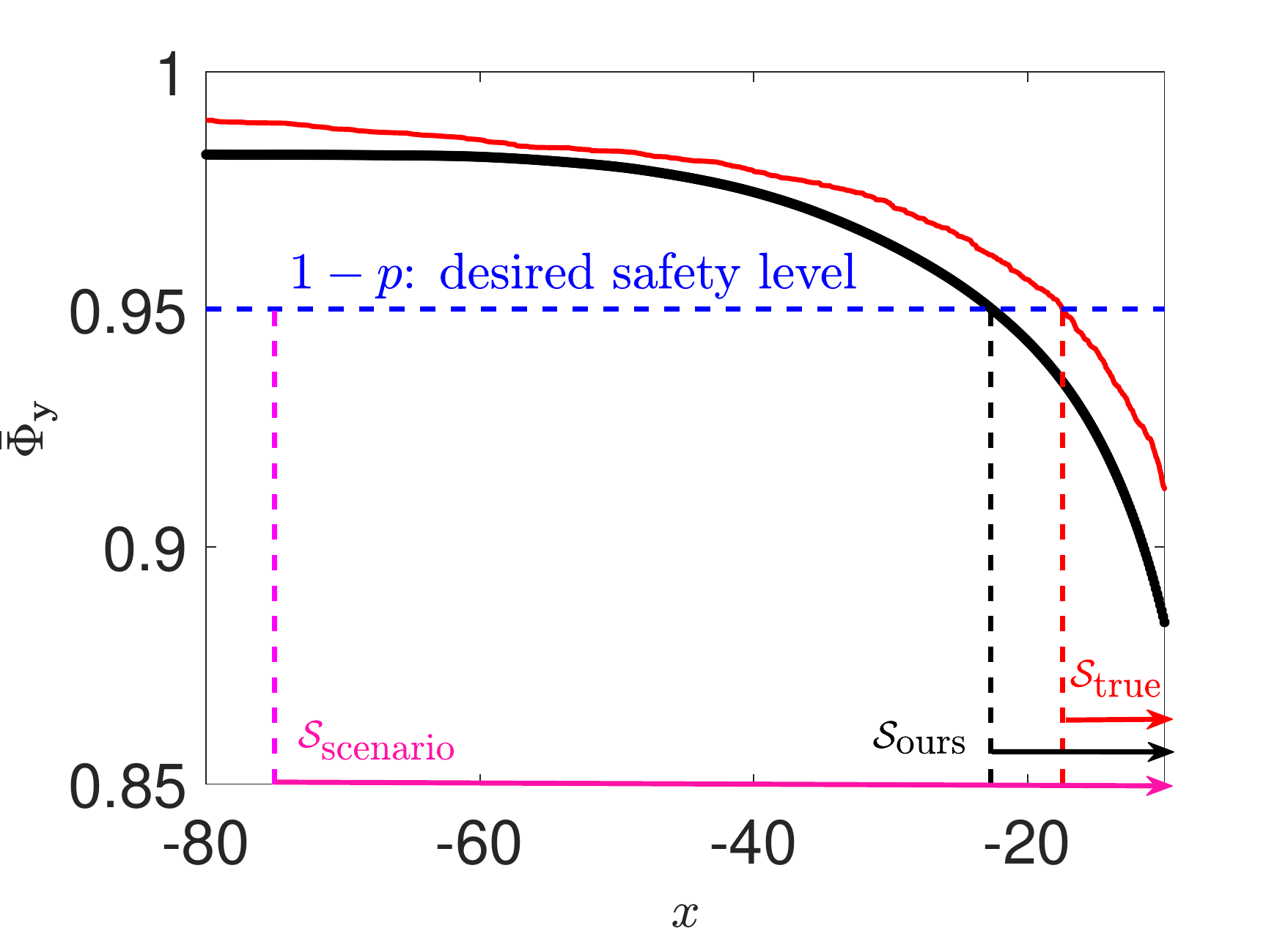}}
    \caption{Comparison of ReLU network safety verification.}
    \label{fig:CDFprobabilities}
\end{figure}

The accuracy of this propagation depends on the grid resolution in the frequency domain and the numerical accuracy of the HT used to propagate the CF through the max layer.
A finer grid in the frequency domain with a large number of terms in the HT summation yields better results than a coarser grid with fewer terms in the summation.
To illustrate this, the fourth column in Table~\ref{tbl:gridResTimes} computes the average error in probability across all trials for various values of the grid resolution and HT parameters, where $\Delta$ represents the difference in the computed probability of success with the given probability threshold, i.e.,
\begin{equation}
    \Delta := \bar{\Phi}_{\mathbf{y}}(0) - (1 - p),
\end{equation}
where $\bar{\Phi}_{\mathbf{y}}(0) := \mathbb{P}(\mathbf{x}^L > 0) = 1 - \Phi_{\mathbf{x}^L}(0)$ is known as the \textit{complementary} CDF.
See Figure~\ref{fig:CDFprobabilities} for a visual representation of these differences.
Thus, the difference in these deltas is a metric for how accurate the CF propagation is --- if the numerics were exact ($M,L\rightarrow\infty$), then $\Delta_{p^*} = \Delta_{p}$.
Note that the trial for Figure~\ref{fig:CDFprobabilities} fails verification because $\Delta < 0$, which implies that the probability of being in the safety set is \textit{less} than $1 - p = 0.95$.

For the trial in Figure~\ref{fig:CDFprobabilities}(a), the estimated probability of being in the safety set is approximately 23.6\%, whereas the true probability is 27.9\%, giving $|\delta\Delta_p| = 4.3\%$.
In addition, we compared against \cite[Appendix D]{9857969}, which uses scenario optimization to solve the reverse problem; namely, that of finding the maximal safety set $\mathcal{X} = \bar{r}(p) = \sup_{r}\{r\in\mathbb{R} : \mathbb{P}(\mathbf{y} > r) \geq 1 - p\}$.
Running the method by choosing samples according to $N\geq\frac{2}{\epsilon}(\log(\frac{1}{\delta})+1)$ where $\delta$ is a confidence parameter such that $\mathbb{P}_{\tilde{r}}\{\mathbb{P}(\mathbf{y} > r) \geq 1 - p\}\geq 1-\delta$.  
With $\delta = 10^{-5}$, we require $501$ samples. 
Over 500 trials, we generated 501 samples and propagated them through the network. 
The best $\tilde{r} = -74.99$ with the average over 500 trials is $\mathbb{E}[\tilde{r}] = -3297.92$, whereas the \textit{true} 95\% quantile occurs at $x^* = -17.43$ while our estimated quantile is at $x = -22.67$.
We mark the quantile values with vertical lines on Figure~\ref{fig:CDFprobabilities}(b).
The Cauchy distribution has a longer tail than the normal distribution, thus sampling from it produces more outliers.
This does not bode well for sampling-based verification methods, causing large over-approximations of the safety set.
\begin{table}[t!]
	\caption{Verification times, approximation errors for different values of Hilbert transform terms ($h, M$), and grid resolution ($N$).}%
	\centering
	\begin{tabular}{ccccc||ccccc}
		\hline
		$h$ & $N$ & $M$ & $\mathbb{E}[|\delta\Delta_p|]$ & Avg Time (s) & $h$ & $N$ & $M$ & $\mathbb{E}[|\delta\Delta_p|]$ & Avg Time (s) \\
		\hline
		1.0 & $10^4$ & 5,000 & 0.0259 & 17.78 & 0.7 & $10^4$ & 2,000 & 0.0254 & 7.26\\
		0.6 & $10^4$ & 5,000 & 0.0238 & 17.52 & 0.7 & $10^4$ & $10^3$ & 0.0303 & 3.55\\
		0.5 & $10^4$ & 5,000 & 0.0209 & 18.69 & 0.7 & $10^3$ & $10^3$ & 0.1007 & 0.29\\
		0.1 & $10^4$ & 5,000 & 0.0228 & 18.55 & 0.7 & $10^3$ & 100 & 0.1009 & 0.03\\
		\hline
	\end{tabular}
\label{tbl:gridResTimes}
\end{table}

Table~\ref{tbl:gridResTimes} also shows the average time it takes to complete verification for one trial for various parameters.
For a grid resolution of $10^4$ points and $10^3$ HT computations per grid point, we can get verification results in approximately 3~sec for a two-layer network.
Naturally, the accuracy of the propagation degrades with lower values for the parameters, but the computation time decreases, so there is a trade-off between accuracy and speed.

\subsection{Example 2}
\begin{figure}[t!]
    \centering
    \includegraphics[width=\textwidth,trim={0.1cm 0 0.1cm 0},clip]{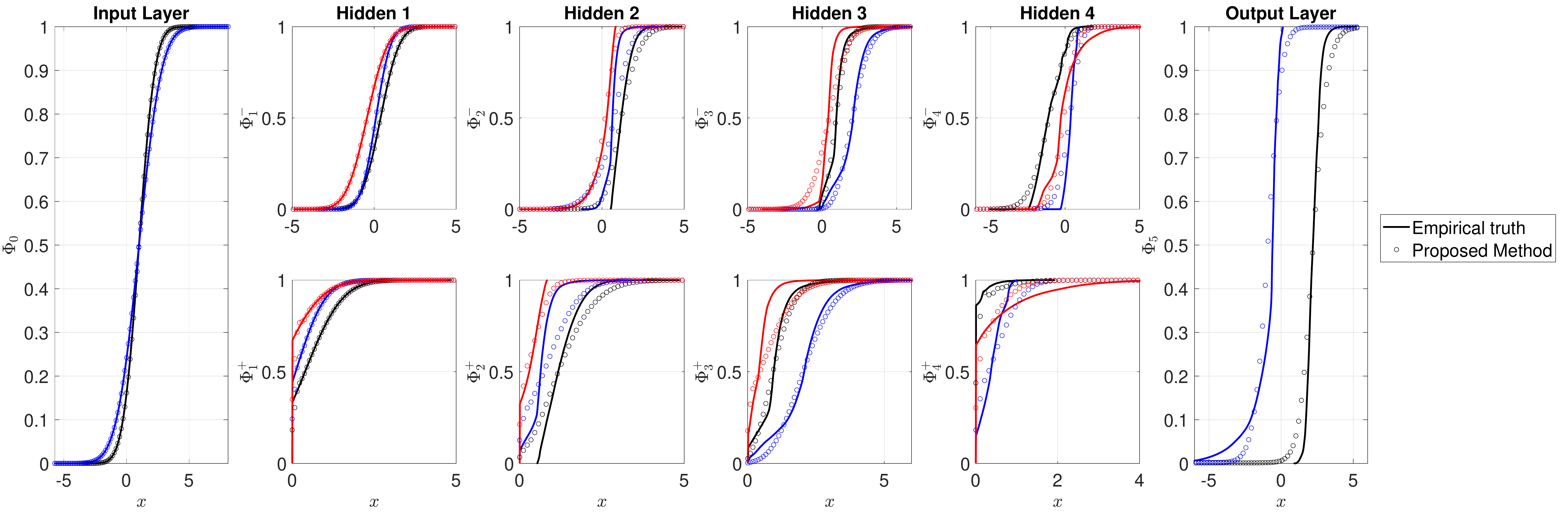}
    \caption{Comparison of empirical truth and estimated CDFs for each layer of the ReLU network where only the CDF of the first three neurons are plotted before activation, $\upvarphi^-$, and after activation, $\upvarphi^+$. The CDF as calculated from the CF via \eqref{eq:cftocdf} (circles) closely matches the empirically calculated CDF from the propagation of $10^4$ samples (solid lines) even for 50 neuron hidden layer deep networks.}
    \label{fig:CDFsDeepNetwork}
\end{figure}

We now consider a more complex network based on~\citep{probVerification_fazlyab}.
In this example, we have 2 inputs, 5 hidden layers, 50 neurons in each of the 5 hidden layers, and 2 outputs.
The inputs are normally distributed with mean $\mu_0 = [1, 1]^\intercal$ and covariance $\Sigma = \textrm{diag}(1,2)$.
We assume the weights and biases for each layer are randomly chosen from $\U[-1,1]$.
For the propagation we use $d = 20$ for the frequency cutoffs, $N = 10^4$ grid points for the frequency resolution, and $h = 0.5$ and $M = 5000$ for the HT computations.

Figure~\ref{fig:CDFsDeepNetwork} shows the evolution of the CDFs of each marginal distribution along the network. 
The labels $\upvarphi^{+/-}$ denote the CDF before and after the ReLU activation layer.
We see that the characteristic function propagation is relatively accurate throughout the whole network given the resolution in the CF and HT. 
The inaccuracies result from the evaluation of the CDF at $x = 0$ as can be first seen in $\Phi_{1}^{+}$.
The sinc method \citep{sincHT} that was used to compute the HT and CDF does not perform very well at discontinuity points and these errors propagate after each max layer.

\section{Conclusion and Future Work}\label{sec:conclusion}

We have presented a probabilistic verification scheme for ReLU neural networks using the machinery of characteristic functions.
We show that our method not only can achieve fast and accurate distribution propagation through a ReLU feedforward (deep) neural network, but also verification becomes a simple evaluation of the network output cumulative density function.
One extension of this work could be to optimize the risk level by minimizing $p$ such that $\mathbb{P}(f(\mathbf{x}^0)\in\mathcal{S})\geq 1-p$, for some input distribution $\mathbf{x}^0\sim\psi^0$.
Moreover, we can consider the reverse problem of finding the \textit{largest} input set $\mathcal{X}$ such that a network is probabilistically safe for a given risk level $p$~\citep{9857969,PROVEN}. 
Additionally, we could also study how the numerical errors in the HT computation propagate throughout the network; optimizing the parameters $M, h$ for each activation layer would greatly improve the accuracy of the propagation.
Lastly, it might be possible to extend this framework to other activation functions, as long as one can analytically propagate the CF through that activation function.

\section*{Acknowledgments}

We thank Brendon G. Anderson for providing us with the code of  \citep{9857969}. 
This work has been supported in part by the National Science Foundation under award CNS-1836900 and
by NASA under the University Leadership Initiative award 80NSSC20M0163.
Any opinions, findings, and conclusions or recommendations expressed in this material are those of the authors and do not necessarily reflect the views of the NSF or any NASA entity.

\bibliography{refs.bib}

\end{document}